\documentclass{article}

\PassOptionsToPackage{numbers}{natbib}
\usepackage[preprint]{neurips_2020}

\usepackage[utf8]{inputenc} % allow utf-8 input
\usepackage[T1]{fontenc}    % use 8-bit T1 fonts
\usepackage{hyperref}       % hyperlinks
\usepackage{url}            % simple URL typesetting
\usepackage{booktabs}       % professional-quality tables
\usepackage{amsfonts}       % blackboard math symbols
\usepackage{nicefrac}       % compact symbols for 1/2, etc.
\usepackage{microtype}      % microtypography
\usepackage[pdftex]{graphicx}

\usepackage{enumitem}
\usepackage{amsmath, amssymb}
\usepackage{bbm}
\usepackage[ruled,vlined]{algorithm2e}
\usepackage{algorithmic}
\usepackage{amsthm}
\usepackage{subfig}
\usepackage{wrapfig}
\usepackage{mathtools}

\usepackage{tikz}
\usetikzlibrary{automata}
\usetikzlibrary{calc, automata, chains, arrows.meta}
\usetikzlibrary{chains,scopes}
\usetikzlibrary{positioning}

\newtheorem{theorem}{Theorem}

\newtheorem{conjecture}{Conjecture}

\newtheorem{definition}{Definition}

\title{Collapsing Bandits and Their Application to Public Health Interventions
}

\author{%
  Aditya Mate\thanks{equal contribution.} \\
  Harvard University\\
  Cambridge, MA, 02138 \\
  \texttt{aditya\_mate@g.harvard.edu} \\
   \And
   Jackson A. Killian$^*$ \\
  Harvard University\\
  Cambridge, MA, 02138 \\
  \texttt{jkillian@g.harvard.edu} \\
  \And
  Haifeng Xu \\
  University of Virginia\\
  Charlottesville, VA, 22903 \\
  \texttt{hx4ad@virginia.edu} \\
  \And
  Andrew Perrault \\
  Harvard University\\
  Cambridge, MA, 02138 \\
  \texttt{aperrault@g.harvard.edu} \\
  \And
  Milind Tambe \\
  Harvard University\\
  Cambridge, MA, 02138 \\
  \texttt{milind\_tambe@harvard.edu} \\
}

\begin{document}

\maketitle

\begin{abstract}
We propose and study Collapsing Bandits, a new restless multi-armed bandit (RMAB) setting in which each arm follows a binary-state Markovian process with a special structure: when an arm is played, the state is fully observed, thus ``collapsing'' any uncertainty, but when an arm is passive, no observation is made, thus allowing uncertainty to evolve. The goal is to keep as many arms in the ``good'' state as possible by planning a limited budget of actions per round. Such Collapsing Bandits are natural models for many healthcare domains in which health workers must simultaneously monitor patients \textit{and} deliver interventions in a way that maximizes the health of their patient cohort. Our main contributions are as follows: (i) Building on the Whittle index technique for RMABs, we derive conditions under which the Collapsing Bandits problem is \textit{indexable}.
% , a standard requirement for the performance guarantees on technique to hold. 
Our derivation hinges on novel conditions that characterize when the optimal policies may take the form of either ``forward'' or ``reverse'' threshold policies. (ii) We exploit the optimality of threshold policies to build fast algorithms for computing the Whittle index, including a closed form. (iii) We evaluate our algorithm on several data distributions including data from a real-world healthcare task in which a worker must monitor and deliver interventions to maximize their patients' adherence to tuberculosis medication. Our algorithm achieves a 3-order-of-magnitude speedup compared to state-of-the-art RMAB techniques, while achieving similar performance. 
% (iv) We implement a Thompson Sampling-based learning approach that performs well with only a small amount of historical data, making our approach amenable to real-world deployment where state transition probabilities would be unknown to start. 
% By introducing Collapsing Bandits, giving theoretical results on the space of possible state transitions, and developing a fast, effective solution algorithm, we hope to enable to more complex planning techniques for several resource-constrained healthcare domains.
\end{abstract}

\section{Introduction}
\textbf{Motivation. } This paper considers scheduling problems in which a planner must act on $k$ out of $N$ binary-state processes each round. The planner fully observes the state of the processes on which she acts, then all processes undergo an action-dependent Markovian state transition; the state of the process is unobserved until it is acted upon again, resulting in uncertainty. The planner's goal is to maximize  the number of processes that are in some ``good'' state over the course of $T$ rounds. This class of problems is natural in the context of \textit{monitoring tasks} which arise in many domains such as sensor/machine maintenance \cite{iannello2012optimality,glazebrook2006some,abbou2019group,villar2016indexability}, anti-poaching patrols \cite{qian2016restless}, and especially healthcare. For example, nurses or community health workers are employed to monitor and improve the adherence of patient cohorts to medications for diseases like diabetes \cite{newman2018community}, hypertension \cite{brownstein2007effectiveness}, tuberculosis \cite{rahedi2014effects,chang2013house} and HIV \cite{kenya2013using,kenya2011can}. Their goal is to keep patients adherent (i.e., in the ``good'' state) but a health worker can only intervene on (visit) a limited number of patients each day. Health workers can play a similar role in monitoring and delivering interventions for patient mental health, e.g., in the context of depression \cite{lowe2004monitoring,mundorf2018reducing} or Alzheimer's Disease \cite{lin2018selective}.

We adopt the solution framework of \textit{Restless Multi-Arm Bandits} (RMABs), a generalization of Multi-Arm Bandits (MABs) in which a planner may act on $k$ out of $N$ arms each round that each follow a Markov Decision Process (MDP). Solving an RMAB is PSPACE-hard in general \cite{papadimitriou1999complexity}. Therefore, a common approach is to consider the Lagrangian relaxation of the problem in which the $\frac{k}{N}$ budget constraint is dualized. Solving the relaxed problem gives Lagrange multipliers which act as a greedy index heuristic, known as the Whittle index, for the original problem.
% Solving the relaxed problem gives Lagrange multipliers for each MDP which intuitively capture the "value for acting" on each MDP in the current round. Then to act in the original constrained problem, one greedily acts on the MDPs with the largest Lagrange multipliers from the solved relaxed problem. This is known as the \textit{Whittle index} approach, where the Whittle index refers to the value of the Lagrangian when the relaxed problem is solved \cite{whittle1988restless}.
The Whittle index approach has been shown to be asymptotically optimal (i.e., $N\rightarrow{}\infty$ with fixed $\frac{k}{N}$) \cite{weber1990index} and performs well empirically \cite{ansell2003whittle} making it a common solution technique for RMABs.

Critically, using the Whittle index approach requires two key components: (i) a fast method for computing the index and (ii) proving the problem satisfies a condition known as \textit{indexability}. Without (i) the approach can be prohibitively slow, and without (ii) performance guarantees are sacrificed. Neither (i) nor (ii) are known for general RMABs. 
% Moreover, to the best of our knowledge, no subclass of RMABs exists in the literature that ensures (1) and (2) and considers POMDP arms with a belief-maximizing objective.
Therefore, to capture the scheduling problems addressed in this work, we introduce a new subclass of RMABs, \textit{Collapsing Bandits}, distinguished by the following feature: when an arm is played, the agent fully observes its state, ``collapsing'' any uncertainty, but when an arm is passive, no observation is made and uncertainty evolves. We show that this RMAB subclass is more general than previous models and leads to new theoretical results, including conditions under which the problem is indexable and under which optimal policies follow one of two simple threshold types. We use these results to develop algorithms for quickly computing the Whittle index. In experiments, we analyze the algorithms' performance on (i) data from a real-world healthcare scheduling task in which our approach ties state-of-the-art performance at a fraction the runtime and (ii) various synthetic distributions, some of which the algorithm achieves performance comparable to the state of the art even outside its optimality conditions.

% In experiments inspired by real-world healthcare scheduling tasks, we demonstrate that the conditions identify a wide range of indexable problems that have threshold-type optimal policies, providing optimality guarantees on the performance of the algorithm. Furthermore,
% % (1) we find no counterexamples to indexability; and (2) 
% even outside of the algorithm's optimality conditions, it achieves performance comparable to the state of the art.
% These observations provide directions for future work.

% We evaluate our work by modeling as a Collapsing Bandit problems inspired by real-world healthcare scheduling tasks. Our experiments include an analysis of tuberculosis (TB) medication adherence monitoring for which simulations are drawn from real data. We also design a tailored learning approach based on Thompson Sampling to adapt our algorithm to the real-world setting in which underlying transition probabilities are initially unknown. We show that not only does our algorithm perform well across different data distributions both when transitions are known and unknown, but our theoretical guarantees cover a large portion of the real-world data distribution.

To summarize, our contributions are as follows:
(i) We introduce a new subclass of RMABs, Collapsing Bandits, (ii) Derive theoretical conditions for Whittle indexability and for the optimal policy to be threshold-type, and (iii) Develop an efficient solution that achieves a 3-order-of-magnitude speedup compared to more general state-of-the-art RMAB techniques, without sacrificing performance.
% (4) Implement a tailored learning approach that learns effectively in real-world settings with small amounts of historical data, making our approach amenable to real-world deployment.

% We implement their approach as a baseline, but show that our specialized subclass allows for explicit derivation of the form of optimal policies which allow us to compute the Whittle index directly, making our approach orders of magnitude faster.
% \citet{jung_neurips_RMABSinThomspson} proves regret bounds for Thompson sampling in episodic RMAB; we do not have episodes. 
% Related to optimization for CHWs, \citet{kunkel2014optimal,bhattacharya2018restless} optimize geographic allocation of CHW groups but do not help individual CHWs plan interventions. 

\section{Restless Multi-Armed Bandits}
\label{section:preliminaries}
An RMAB consists of a set of $N$ arms, each associated with a \emph{two-action} MDP \cite{puterman2014markov}. An MDP $\{ \mathcal{S}, \mathcal{A}, r, P\}$ consists of a set of states $\mathcal{S}$, a set of actions $\mathcal{A}$, a state-dependent reward function $r: \mathcal{S} \rightarrow \mathbb{R}$, and a transition function $P$, where $P^a_{s,s^\prime}$ denotes the probability of transitioning from state $s$ to $s^\prime$ when action $a$ is taken.  An MDP \emph{policy} $\pi: \mathcal{S} \rightarrow \mathcal{A}$ represents a choice of action to take at each state. We will consider both discounted and average reward criteria. 
The long-term \emph{discounted reward} starting from state $s_0 = s$ is defined as $R_{\beta}^\pi(s) = E\left[\sum_{t=0}^\infty \beta^tr(\pi(s_t))|\pi, s_0 = s\right]$ where $\beta \in [0,1)$ is the discount factor and actions are selected using $\pi$. To define average reward, let $f^\pi(s): \mathcal{S} \rightarrow [0,1]$ denote the \emph{occupancy frequency} induced by policy $\pi$, i.e., the fraction of time spent in each state of the MDP. The \emph{average reward} $\overline{R}^\pi$ of policy $\pi$ be defined as the expected reward computed over the occupancy frequency: $\overline{R}^\pi = \sum_{s \in \mathcal{S}} f^\pi(s) r(s)$.
% Computing a policy that maximizes $R_\beta(s)$ or $\overline{R}^\pi$ can be done in polynomial time in the state and action spaces \citep{puterman2014markov}.
% We use $\Pi^*$ to denote the set of all optimal policies.

Each arm in an RMAB is an MDP with the action set $\mathcal{A}=\{0,1\}$. Action $1$ ($0$) is called the \emph{active} (\emph{passive}) action and denotes the arm being pulled (not pulled).  The agent can pull at most $k$ arms at each time step. The agent's goal is to maximize either her discounted or average reward across the arms over time. 
Some RMAB problems need to account for partial observability of states. It is sufficient to let the MDP state be the \emph{belief state}: the probability of being in each latent state \citep{kaelbling1998planning}. While intractable in general due to infinite number of reachable belief states, most partially observable RMABs studied (including our Collapsing Bandits) have polynomially many belief states due to a finite time horizon or other structures.

\textbf{Related work} 
RMABs have been an attractive framework for studying various stochastic scheduling problems since Whittle indices were introduced \cite{whittle1988restless}. Because general RMABs are PSPACE-hard \cite{papadimitriou1999complexity}, RMAB studies usually consider restricted classes under which some performance guarantees can be derived. Collapsing Bandits form one such novel class that generalizes some existing results which we note in later sections.
% We demonstrate the novelty of Collapsing Bandits by comparing to other restricted classes studied previously in the literature. Though some extensions to the Whittle index have been studied based on partial conservation laws \cite{nino2001restless} and marginal productivity \cite{nino2009restless,nino2007dynamic}, we consider the Whittle index directly, and thus review relevant literature. 
\citet{zhao_liyu_paper} develop an efficient Whittle index policy for a 2-state partially observable RMAB subclass in which the state transitions are unaffected by the actions taken and reward is accrued from the active arms only.
% Our work considers action-dependent state transitions as well as reward from active and passive arms.
% \citet{qing_zhao_learning} prove a no-regret learning algorithm for RMABs when rewards are accrued from the active arms only, but we consider the more difficult setting which allows both passive and active rewards.
\citet{akbarzadeh2019restless} define a class of bandits with ``controlled restarts,'' giving indexability results and a method for computing the Whittle index. However, ``controlled restarts'' define the active action as state independent, a stronger assumption than Collapsing Bandits which allow state-dependent action effects.
\citet{glazebrook2006some} give Whittle indexability results for three classes of restless bandits: (1) A machine maintenance regime with deterministic active action effect (we consider stochastic active action effect) (2) A switching regime in which the passive action freezes state transitions (in our setting, states always change regardless of action) (3) A reward depletion/replenishment bandit which deterministically resets to a start state on passive action (we consider stochastic passive action effect).
\citet{AOI_paper1} and \citet{sombabu_paper} augment the machine maintenance problem from \citet{glazebrook2006some} to include either i.i.d.~or Markovian evolving probabilities of an active action having no effect, a limited form of state-dependent action.
% However, this ``hit-or-miss'' action mechanism is not applicable to our setting in which the state is always "collapsed" on the active action.
\citet{meshram2018whittle} introduce Hidden Markov Bandits which, similar to our approach, consider binary state transitions under partial observability, but do not allow for state dependent rewards on passive arms. In sum, our Collapsing Bandits introduce a new, more general RMAB formulation than special subclasses previously considered. \citet{qian2016restless} present a generic approach for any indexable RMAB based on solving the (partially observable) MDPs on arms directly. Because we derive a closed form for the Whittle index, our algorithm is orders of magnitude faster.

\section{Collapsing Bandits}
\label{section:problem_formulation}
We introduce \emph{Collapsing Bandits} (CoB) as a specially structured RMAB with partial observability. In CoB, each arm $n\in\{1,\ldots,N\}$ has binary latent states $\mathcal{S}=\{0, 1\}$, representing \emph{bad} and \emph{good} state, respectively. The agent acts during each of finite days $t \in 1, \ldots, T$.
Let $a_t \in \{0,1\}^N$ denote the vector of actions taken by the agent on day $t$. Arm $n$ is said to be \emph{active} at $t$ if $a_t(n)=1$ and \emph{passive} otherwise. The agent acts on $k$ arms per day, i.e., $\left\|a_t\right\| = k$, where $k \ll N$ because resources are limited. When acting on arm $n$, the true latent state of $n$ is fully observed by the agent and thus its uncertainty ``collapses'' to a realization of the binary latent state. We denote this observation as $\omega \in \mathcal{S}$. States of passive arms are completely unobservable by the agent.

Active arms transition according to the \emph{transition matrix} ${P}_{s,s'}^{a,n}$ and passive arms transition according to $P_{s,s'}^{p,n}$. We drop the superscript $n$ when there is no ambiguity. Our scheduling problem, like many problems in analogous domains, exhibits the following natural structure: (i) processes are more likely to stay ``good'' than change from ``bad'' to ``good''; (ii) when acted on, they tend to improve. These natural structures are respectively captured by imposing the following constraints on $P^p$ and $P^a$ for each arm: (i)~$P_{0,1}^p < P_{1,1}^p$ and $P_{0,1}^a < P_{1,1}^a$; (ii)~$P_{0,1}^{p} < P_{0,1}^{a}$ and $P_{1,1}^{p} < P_{1,1}^{a}$. To avoid unnecessary complication through edge cases, all transition probabilities are assumed to be nonzero.
% The agent aims to maximize the sum of rewards across the arms. 
The agent receives reward $r_t = \sum_{n=1}^{N}s_t(n)$ at $t$, where  $s_t(n)$ is the latent state of arm $n$ at $t$.
% and is a Bernoulli random variable. %, i.e., 1 if the state is "good" and 0 otherwise. 
The agent's goal is to maximize the long term rewards, either discounted or average, defined in Sec.~\ref{section:preliminaries}.
% $R=\sum_{t=1}^{T}\mathbb{E}[r_t]$.

% \subsection{Example Applications of Collapsing Bandits}
% \hf{need to have such a section somewhere --- in the intro or here} 

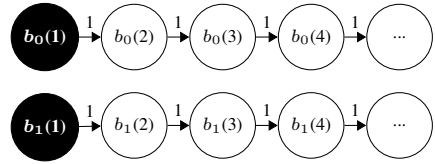
\begin{wrapfigure}{r}{0.45\textwidth}
\centering
\resizebox{!}{60pt}{%
\begin{tikzpicture}[
    % start chain = going right,
    -Triangle, every loop/.append style = {-Triangle},
    start chain=main going right,
    state/.style={circle,minimum size=9mm, draw},
    % every node/.style=draw,
      % minimum size=2cm,
      % max size=5mm,
      node distance=3mm,
      font=\scriptsize,
      >=stealth,
      auto
    ]
  \node [state, on chain, fill=black, text=white] (1) {\textbf{\boldmath $b_0$(1)}};
   {[start branch=A going below]
   \node[state, on chain, fill=black, text=white]  (A) {\textbf{\boldmath $b_1$(1)}};
   }
   \node [state, on chain] (2) {$b_0$(2)};
   {[start branch=A going below]
   \node[state, on chain]  (B) {$b_1$(2)};
   }
    \node [state, on chain] (3) {$b_0$(3)};
   {[start branch=A going below]
   \node[state, on chain]  (C) {$b_1$(3)};
   }
  \node[state, on chain]  (4) {$b_0$(4)};
  {[start branch=A going below]
   \node[state, on chain]  (D) {$b_1$(4)};
   }
   \node[state, on chain]  (5) {...};
  {[start branch=A going below]
   \node[state, on chain]  (E) {...};
   }
%   \node[state, on chain]  (L) {$b_0$(L))};
%   {[start branch=A going below]
%   \node[state, on chain]  (E) {$b_1$(L)};
%   }

 \path[]
    (1) edge node[above] {1} (2)
    (2) edge node[above] {1} (3)
    (3) edge node[above] {1} (4)
    (4) edge node[above] {1} (5);

\path[]
    (A) edge node[above] {1} (B)
    (B) edge node[above] {1} (C)
    (C) edge node[above] {1} (D)
    (D) edge node[above] {1} (E);
  
\end{tikzpicture}
}
\caption{Belief-state MDP under the policy of always being passive. There is one chain for each observation $\omega \in \{0,1\}$ with the head marked black. Belief states deterministically transition down the chains.}
\label{fig:bsMDP}
\end{wrapfigure}
% Our derivations use features of the belief-state MDP induced by CoB problems. In this section, we describe the reduction to belief-state MDP and give conditions for its Whittle indexability.

\paragraph{Belief-State MDP Representation}
% Our theoretical derivations in the next section use features of the belief-state MDP induced by CoB problems; we describe the reduction to belief-state MDP here. 
In limited observability settings, belief-state MDPs have organized chain-like structures, which we will exploit. In particular, the only information that affects our belief of an arm being in state $1$ is the number of days since that arm was last pulled and the state $\omega$ observed at that time. Therefore, we can arrange these belief states into two ``chains'' of length $T$, each for an observation $\omega$. A sketch of the belief state chains under the passive action is shown in Fig.~\ref{fig:bsMDP}. Let $b_\omega(u)$ denote the belief state, \emph{i.e., the probability that the state is $1$}, if the agent received observation $\omega \in \{ 0,1\} $ when it acted on the process $u$ days ago. Note that $b_\omega(u)$ is also the expected reward associated with that belief state, and let $\mathcal{B}$ be the set of all belief states.

When the belief-state MDP is allowed to evolve under some policy, the following mechanism arises: first, after an action, the state $\omega$ is observed (uncertainty ``collapses''), then one round passes causing the agent's belief to become $P_{\omega,1}^a$, representing the head of the chain determined by $\omega$. Subsequent passive actions cause the process to transition deterministically down the same chain (though, the transition in the latent state is still stochastic). Then when the process's arm is active, it transitions to the head of one of the chains with probability equal to the belief that the corresponding observation would be emitted (see Fig.~\ref{fig:chains} for an illustration).

The belief associated with a belief state can be calculated in closed form with the given transition probabilities. Formally,
\begin{small}
\begin{align}
\label{eq:tau_maintext}
b_{\omega}(u) = \tau_{u-1}(P_{\omega,1}^a) \text{   } \forall u \in [T]\hspace{1mm}\text{where}\hspace{1mm}
    % \tau(b) = P_{1,1}^p b + (1-b)P_{0,1}^p
    \tau_u (b)= \frac{P_{0,1}^p - (P_{1,1}^p-P_{0,1}^p)^u(P_{0,1}^p- b(1+P_{0,1}^p-P_{1,1}^p))}{(1+P_{0,1}^p-P_{1,1}^p)}
\end{align}
\end{small}
% And $u$ gives the number of applications of $\tau$.

% We call the chain of belief states that begins with observing the process to be in the ``good'' state the \emph{top chain} and the other chain the \emph{bottom chain}.

% So the state/belief/reward at the head of each chain is $b_\omega(1) = P_{\omega,1}^a$. This is because in our model, unlike traditional POMDPs, the observations are generated by the state \textit{before} the action is taken, rather than after. 

\section{Collapsing Bandits: Threshold Policies and Whittle Indexability}
% \section{Fast Computation of the Whittle Index}
\label{section_computing_whittle_index}
% We now discuss computing the Whittle index, which for general problems can be difficult. Qian et al.~\cite{qian2016restless} give two general methods for dealing with POMDP arms: one for computing indices to an $\epsilon$-precision, and one for ordering arms by index without exact computation. However, both methods rely on solving a number of POMDPs per round that is at least psuedo-polynomial in N. We would like to develop a faster approach that is linear in N and that avoids the computationally intensive step of solving POMDPs. To that end, we consider a subclass of monotone or \textit{threshold} policies, and leverage their structure to develop a fast closed form computation of the Whittle index. 

% Thm.~\ref{thm:indexability} says that the Whittle index policy for the adherence monitoring problem has theoretical guarantees \cite{weber1990index} and is likely to have strong empirical performance. However, even with a small belief-state MDP for each patient ($|\Omega|T$ states and 2 actions), computing the Whittle index by performing a binary search over the subsidy value $m$ is relatively expensive (this method was proposed by Qian et al.~\citet{qian2016restless}). Recall that our objective is to find an $m$ such that exactly $k$ processes are selected to be active. Evaluating a candidate $m$ requires solving \emph{all} of the corresponding MDPs.

% To facilitate a close form computation of the Whittle index, we show that \emph{forward threshold} policies are optimal for a wide class of Collapsing Bandits.
% As mentioned in the introduction, due to the known 
Because of the well-known intractability of solving general RMABs, the widely adopted solution concept in the literature of RMABs is the Whittle index approach; for a comprehensive description, see \citet{whittle1988restless}. Intuitively, the Whittle index captures the value of acting on an arm in a particular state by finding the minimum \emph{subsidy} $m$ the agent would accept to \textit{not act}, where the subsidy is some exogenous ``donation'' of reward. Formally, the modified reward function becomes $r_m: \mathcal{S}\times\mathcal{A} \rightarrow \mathbb{R}$, where $r_m(s,0) = r(s) + m$ and $r_m(s,1) = r(s)$. Let $R_{\beta,m}^\pi(s) = E\left[\sum_{t=0}^\infty \beta^tr_m(s_t,\pi(s_t))|\pi, s_0 = s\right]$ and $\overline{R}^\pi_m = \sum_{s \in \mathcal{S}} f^\pi(s) r_m(s,\pi(s))$ be the discounted and average reward criteria for this new subsidy setting, respectively. The former is maximized by the discounted value function (we give a value function for the average reward criterion in \textbf{Fast Whittle Index Computation}):
\begin{equation}
    \label{eq:discounted_value_fn_definition}
    \begin{split}
        V_m(b) = \max
        \begin{cases}
        m+b+\beta V_m(\tau_1(b)) & \text{passive} \\
        b + \beta(bV_m(P_{1,1}^a) + (1-b)V_m(P_{0,1}^a)) & \text{active}
        \end{cases}
    \end{split}
\end{equation}
where $\tau$ is defined in Eq.~\ref{eq:tau_maintext} and $b$ is shorthand for $b_\omega(u)$. In a CoB, the Whittle index of a belief state $b$ is the smallest $m$ s.t.~it is equally optimal to be active or passive in the current state. Formally:
\begin{align}\label{whittle_subproblem}
W(b) = \inf_m\{m : V_{m}(b;a=0) \ge V_m(b;a=1)\}
\end{align} Critically, performance guarantees hold only if the problem satisfies \emph{indexability} \cite{weber1990index,whittle1988restless}, a condition which says that for all states, the optimal action cannot switch to active as $m$ increases. Let $\Pi^*_m$ be the set of policies that maximize a given reward criterion under subsidy $m$.
\begin{definition}[Indexability]
\label{def:indexability}
An arm is indexable if $\mathcal{B}^*(m) = \{b : \forall\pi~\in~\Pi^*_m, \pi(b)=0\}$ monotonically  increases from $\emptyset$ to the entire state space as $m$ increases from $-\infty$ to $\infty$. An RMAB is indexable if every arm is indexable.
\end{definition}
The following special type of MDP policy is central to our analysis. 
\begin{definition}[Threshold Policies] 
% Suppose the states $b\in\mathcal{B}$ are ordered descending (ascending). 
A policy is a \emph{forward (reverse) threshold policy} if there exists a threshold $b_{th}$ such that $\pi(b) = 0$ ($\pi(b) = 1$) if $b>b_{th}$ and $ \pi(b)=1$ ($\pi(b)=0$) otherwise.
\label{def:threshold_pols}
\end{definition}

\begin{theorem}\label{thm:indexability_maintext}
If for each arm and any subsidy $m \in \mathbb{R}$, there exists an optimal policy that is a forward or reverse threshold policy, the Collapsing Bandit is indexable under discounted and average reward criteria.
% When forward or reverse threshold policies are optimal for any subsidy $m$, Collapsing Bandit problems are indexable under discounted and average reward criteria.
\end{theorem}
\begin{proof}[Proof Sketch]
% \hf{This needs to be re-written. Current writing makes it sound a trivial proof. The proof sketch should be the place for you to highlight the proof challenges/difficulties and how you overcome it. }

Using linearity of the value function in subsidy $m$ for any fixed policy, we first argue that when forward (reverse) threshold policies are optimal, proving indexability reduces to showing that the threshold monotonically decreases (increases) with $m$. Unfortunately, establishing such a monotonic relationship between the threshold and $m$ is a well-known challenging task in the literature that often involves problem-specific reasoning \cite{zhao_liyu_paper}. Our proof features a sophisticated induction argument exploiting the finite size of $\mathcal{B}$ and relies on tools from real analysis for limit arguments.

\end{proof}
All formal proofs can be found in the appendix. We remark that Thm.~\ref{thm:indexability_maintext} generalizes the result in the seminal work by \citet{zhao_liyu_paper} who proved the indexability for a special class of CoB. In particular, the RMAB in \citet{zhao_liyu_paper} can be viewed as a CoB setting with $P^a = P^p$, i.e., transitions are independent of actions.

Though the Whittle index is known to be challenging to compute in general \cite{whittle1988restless}, we are able to design an algorithm that computes the Whittle index efficiently assuming the  optimality of threshold policies, which we now describe.

\paragraph{Fast Whittle Index Computation}
\label{section:algorithm}
% \begin{wrapfigure}{r}{0.3\textwidth}
% \centering
% \includegraphics[width=0.3\textwidth]{figures/nib_v_ib.pdf}
% \caption{(top) Non-increasing belief process (NIB, top) always has decreasing belief with increasing passive days regardless of observation. Increasing belief process (IB, bottom) has increasing belief with increasing passive days after being observed in the "bad" state. }%
% \label{fig:NIB}
% \end{wrapfigure}

% The previous suggests that either forward or reverse threshold policies are virtually always optimal for any CoB. 
The main algorithmic idea we use is the Markov chain structure that arises from imposing a \textit{forward} threshold policy on an MDP. A forward threshold policy can be defined by a tuple of the first belief state in each chain that is less than or equal to some belief threshold $b_{th} \in [0, 1]$. In the two-observation setting we consider, this is a tuple $(X_0^{b_{th}}, X_1^{b_{th}})$, where $X_\omega^{b_{th}} \in 1,\ldots,T$ is the index of the first belief state in each chain where it is optimal to act (i.e., the belief is less than or equal to $b_{th}$). We now drop the superscript $b_{th}$ for ease of exposition. See Fig.~\ref{fig:chains} for a visualization of the transitions induced by such an example policy. For a forward threshold policy $(X_0, X_1)$, the occupancy frequencies induced for each state $b_\omega(u)$ are:
\begin{align}
f^{(X_0,X_1)}(b_\omega(u))=
\begin{cases}
\alpha & \textrm{if }\omega=0,u \le X_0 \\
\beta  & \textrm{if }\omega=1,u \le X_1 \\
0 & \textrm{otherwise}
\end{cases} \\
\alpha = \bigg(\frac{(X_1 b_0(X_0))}{1-b_1(X_1)} + X_0\bigg)^{-1} \text{,   } \beta = \bigg(\frac{X_1 b_0(X_0)}{1-b_1(X_1)} + X_0\bigg)^{-1}  \frac{b_0(X_0)}{1-b_1(X_1)}
\end{align}
These equations are derived from standard Markov chain theory. These occupancy frequencies do not depend on the subsidy. Let $J_m^{(X_0,X_1)}$ be the average reward of policy $(X_0,X_1)$ under subsidy $m$. We decompose the average reward into the contribution of the state reward and the subsidy
\begin{align}\label{eqn:Ravg}
    J_m^{(X_0,X_1)} = \sum_{b\in \mathcal{B}}bf^{(X_0,X_1)}(b) +
    m(1-f^{(X_0,X_1)}(b_1(X_1))-f^{(X_0,X_1)}(b_0(X_0)))
 \end{align}
Recall that for any belief state $b_\omega(u)$, the Whittle index is the smallest $m$ for which the active and passive actions are both optimal. Given forward threshold optimality, this translates to two corresponding threshold policies being equally optimal. Such policies must have adjacent belief states as thresholds, as can be concluded from Lemma 1 %\ref{lemma:m_star_is_decreasing_function} 
in Appendix A.
%\ref{Appendix:indexability}.
% Given forward threshold optimality, the linearity in $m$ of $J_m(.)$ and Lemma \ref{lemma:m_star_is_decreasing_function} in Appendix \ref{Appendix:indexability} imply that if any two thresholds are equally optimal for some $m$, then these must be adjacent or all thresholds between the non-adjacent optimal thresholds must be optimal.
Note that for a belief state $b_0(X_0)$ the only adjacent threshold policies with active and passive as optimal actions at $b_0(X_0)$ are $(X_0,X_1)$ and $(X_0+1,X_1)$ respectively. Thus the subsidy which makes these two policies equal in value must thus be the Whittle Index for $b_0(X_0)$, which we obtain by solving:  $J_m^{(X_0, X_1)} = J_m^{(X_0+1, X_1)}$ for $m$. We use this idea to construct two fast Whittle index algorithms. 

\begin{figure}[t!]
\centering
\subfloat[]{
\resizebox{!}{72pt}{%
\begin{tikzpicture}[
    % start chain = going right,
    -Triangle, every loop/.append style = {-Triangle},
    start chain=main going right,
    state/.style={circle,minimum size=9mm,draw},
    % every node/.style=draw,
      % minimum size=2cm,
      % max size=5mm,
      node distance=6mm,
      font=\scriptsize,
      >=stealth,
      bend angle=28,
      auto
    ]
  \node [state, on chain, fill=black, text=white] (1) {\textbf{\boldmath $b_0$(1)}};
   {[start branch=A going below]
   \node[state, on chain, fill=black, text=white]  (A) {\textbf{\boldmath$b_1$(1)}};
   }
   \node [state, on chain] (2) {$b_0$(2)};
   {[start branch=A going below]
   \node[state, on chain]  (B) {$b_1$(2)};
   }
    \node [state, on chain] (3) {$b_0$(3)};
   {[start branch=A going below]
   \node[state, on chain, fill=lightgray]  (C) {$b_1$(3)};
   }
  \node[state, on chain, fill=lightgray]  (4) {$b_0$(4)};
  {[start branch=A going below]
   \node[state, on chain]  (D) {$b_1$(4)};
   }
   \node[state, on chain]  (5) {...};
  {[start branch=A going below]
   \node[state, on chain]  (E) {...};
   }
%   \node[state, on chain]  (L) {$b_0$(L))};
%   {[start branch=A going below]
%   \node[state, on chain]  (E) {$b_1$(L)};
%   }

  \foreach \i in {1,...,2} {
     \draw let \n1 = { int(\i+1) } in
      (\i)  edge[] (\n1);
    %   (\n1) edge[bend left] (\i);
  }
  \draw (3)  edge[] (4);
%   \draw (i)  edge[] (L);

  \draw (A) edge[](B);
  \draw (B) edge[] (C);
%   \draw (C) edge[] (D);
%   \draw (D) edge[] (E);
  
  \path[->,draw, thick]
    (C) edge node[near end] {$1-b_1(3)$} (1);
    
    \path[->,draw, thick, bend left=25]
    (C) edge node[near start] {$b_1(3)$} (A);
    
    \path[->,draw, thick, bend right=25]
    (4) edge node[near end, above] {$1-b_0(4)$} (1);
    
    \path[->,draw, thick]
    (4) edge node[near start] {$b_0(4)$} (A);
  
%   \foreach \i in {1,...,2} {
%      \draw (\i)  edge[bend right] (0);
%   }
%   \draw (i)  edge[bend right] (0);
%   \draw (L)  edge[bend right] (0);
   
%   \foreach \i in {2,...,3}
%   \draw    (1)  edge[loop left]   (1);
\end{tikzpicture}
}\label{fig:chains}
}
\hspace{10mm}
\subfloat[]{{\includegraphics[width=0.31\textwidth]{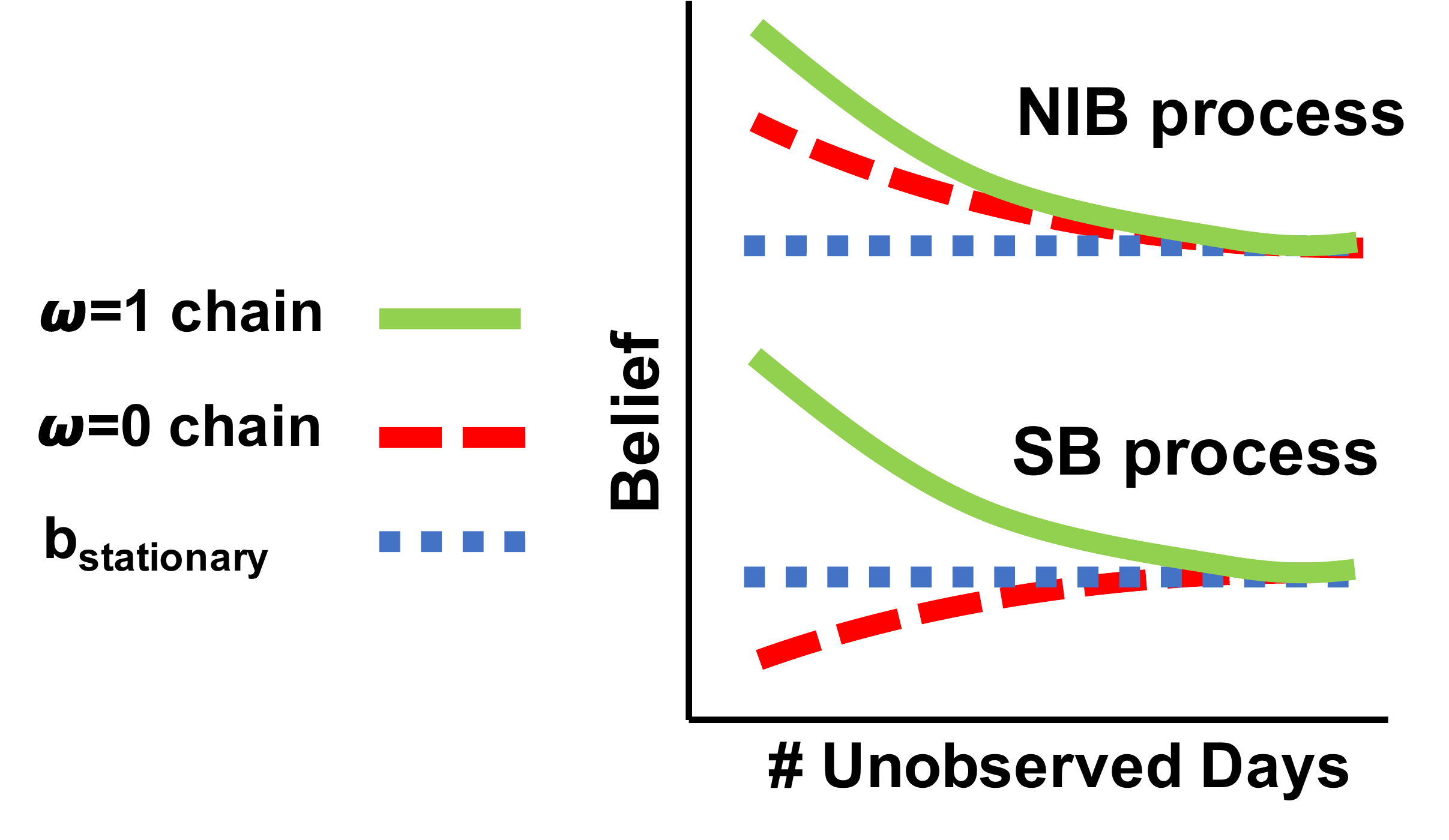}}\label{fig:nib}}%

\caption{(a) Visualization of forward threshold policy ($X_0=4$,$X_1=3$). Black nodes are the head of each chain and grey nodes are the thresholds. (b) Non-increasing belief (NIB) process has non-increasing belief in both chains. A split belief process (SB) has non-increasing belief after being observed in state $1$, but non-decreasing belief after being observed in state $0$.}%
\end{figure}
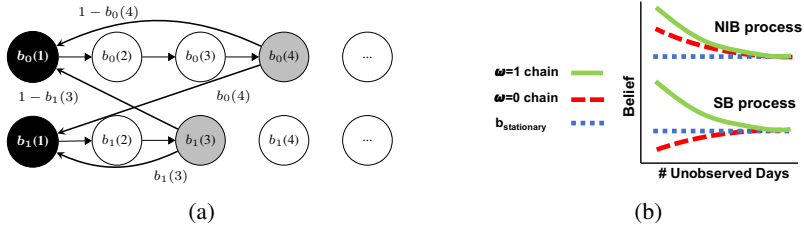

\paragraph{Sequential index computation algorithm} Alg.~\ref{alg:algo1} precomputes the Whittle index of every belief state for each process and has time complexity $\mathcal{O}(|\mathcal{S}|^2T)$ per process.
% \footnote{Note that Alg.\ref{alg:algo1} generalizes to settings where $X_0$ and $X_1$ are not correlated in the optimal policy (in CoBs they are correlated via their relationship to the belief threshold $b_{th}$).}
It is optimized for settings in which the Whittle index can be precomputed. 
% We optimize it further for our problem by harnessing the insight that $\arg\min_{j}\{m_j\}=\arg\max_{j}\{b(X_j)\}$.
However, for online learning settings, we give an alternative method in Appendix F %\ref{appendix:online_learning} 
that computes the Whittle index on-demand, in a closed form.

\begin{algorithm}[h!]
% \begin{algorithmic}[1]
\SetAlgoLined
  Initialize counters to heads of the chains: $X_1 = 1$, $X_0 = 1$ \\
 \While{$X_1 < T$ or $X_0 < T$}
 {
    Compute $ m_1 := m$ such that $  J_{m}^{(X_0,X_1)}=J_{m}^{(X_0,X_1+1)}$ \\
    Compute $m_0 := m$ such that $
    J_{m}^{(X_0,X_1)}=J_{m}^{(X_0+1,X_1)}$ \\
    Set $i=\arg\min\{m_0, m_1\}$ and $W(X_i) = \min \{m_0, m_1 \}$ \\
    Increment $X_i$
   }
%   \end{algorithmic}
\caption{Sequential index computation algorithm \label{alg:algo1}}
\end{algorithm}
Our algorithm also requires that belief is decreasing in $X_0$ and $X_1$. Formally, we require:
\begin{definition}[Non-increasing belief (NIB) processes]
A process has \emph{non-increasing belief} if, for any $u \in [T]$ and for any $\omega \in \mathcal{S}$, $b_\omega(u)\ge b_\omega(u+1)$.
\label{def:NIB}
\end{definition}
All possible CoB belief trends are shown in Fig.~\ref{fig:nib} (full derivation omitted for space). We make this distinction because the computation of the Whittle index in Alg.~\ref{alg:algo1} is guaranteed to be exact for NIB processes that are also forward threshold optimal, though we show empirically that our approach works surprisingly well for most distributions. In the next section, we analyze the possible forms of optimal policies to find conditions under which threshold policies are optimal.

\paragraph{Types of Optimal Policies}
\begin{wrapfigure}{r}{0.5\textwidth}
\centering
\includegraphics[width=0.40\textwidth]{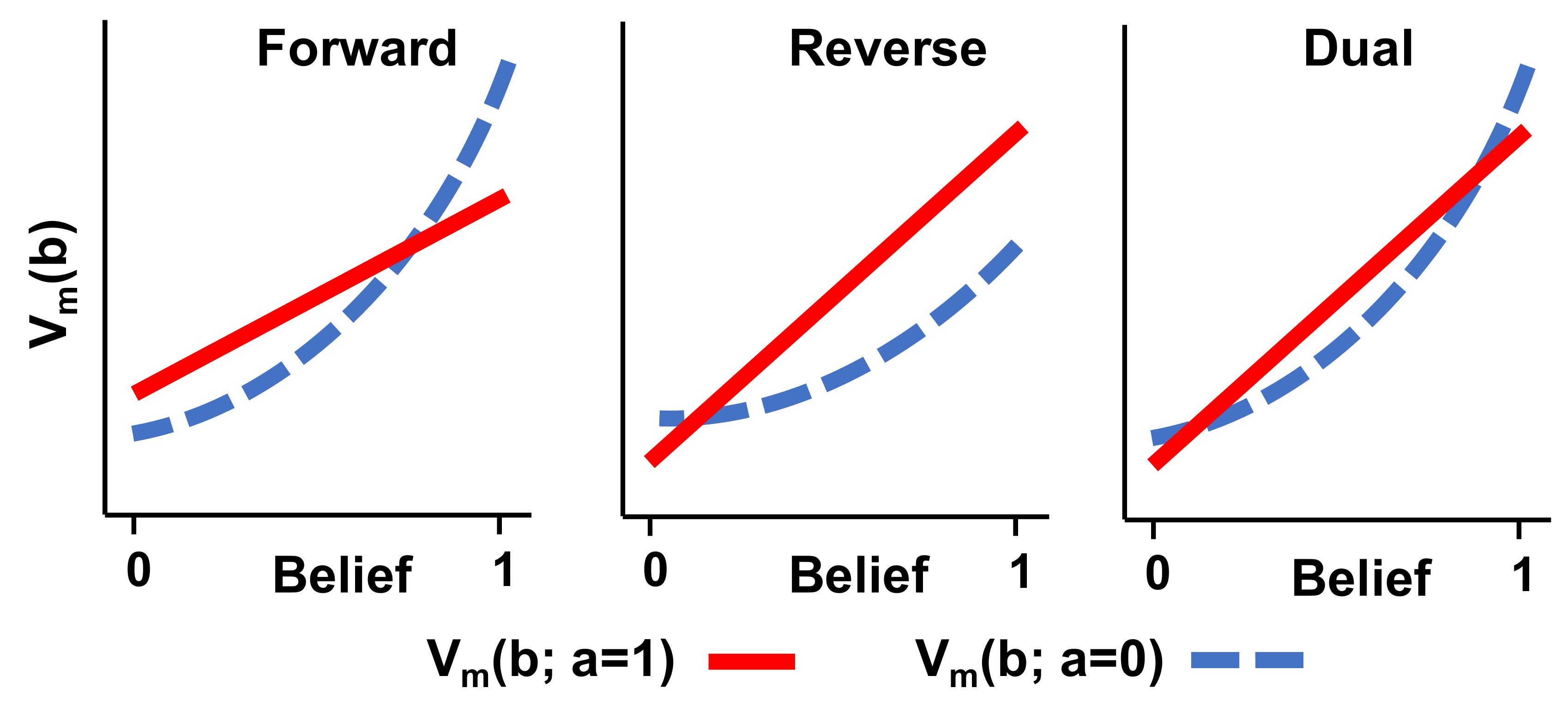}
\caption{Components of $V_m(b)$ in Eq.~\ref{eq:discounted_value_fn_definition}. Since the passive action is convex in $b$, active action is linear in $b$, and value function is a max over these, at most three optimal policy types are possible.}%
\label{fig:value_func_forms}
\end{wrapfigure}
Analyzing Eq.~\ref{eq:discounted_value_fn_definition} reveals that at most three types of optimal policies exist. This follows directly from the definition of $V_m(b)$, which is a max over the passive action value function and the active action value function. The former is convex in $b$, a well-known POMDP result \cite{sondik1978optimal}, and the latter is linear in $b$. Thus, as shown in Fig.~\ref{fig:value_func_forms}, there are three ways in which the value functions of each action may intersect; this defines three optimal policy forms of \textit{forward}, \textit{reverse} and \textit{dual} threshold types, respectively.
Forward and reverse threshold policies are defined in Def.~\ref{def:threshold_pols}; dual threshold policies are active between two separate threshold points and passive elsewhere. Not only do threshold policies greatly reduce the optimization search space, they often admit closed form expressions for the index as demonstrated earlier in this section. We now derive sufficient conditions on the state transition probabilities under which each type of policy is verifiably optimal.

\begin{theorem}\label{thm:forward_threshold_opt}
Consider a belief-state MDP corresponding to an arm in a Collapsing Bandit. For any subsidy $m$, there is a \emph{forward} threshold policy that is optimal under the condition:
\begin{align}
    \label{eq:final_forward_threshold_condition}
    (P_{1,1}^p-P_{0,1}^p)(1+\beta(P_{1,1}^a-P_{0,1}^a))(1-\beta) \geq P_{1,1}^a-P_{0,1}^a
\end{align}
\end{theorem}
\begin{proof}[Proof Sketch]
% \hf{Again, this sounds trivial! Highlight the major difficulties you overcome, and how. }
Forward threshold optimality requires that if the optimal action at a belief $b$ is passive, then it must be so for all $b'>b$. This can be established by requiring that the derivative of the passive action value function is greater than the derivative of the active action value function w.r.t. $b$. The main challenge is to distill this requirement down to measurable quantities so the final condition can be easily verified. We accomplish this by leveraging properties of $\tau(b)$ and using induction to derive both upper and lower bounds on $V_m(b_1)-V_m(b_2) \ \forall \ b_1,b_2\ $ as well as a lower bound on $\frac{d(V_m(b))}{db}$.
% Proving that forward threshold policies are optimal is equivalent to proving that, if it is optimal to act at the current state, then it is optimal to act for all lower beliefs. Formally, if for a belief $b$, the optimal action is to act, then we must show that for a lower $b'<b$, the optimal action is also to act. We do this by showing that the derivative w.r.t.~$b$ of the passive action value function is greater than the derivative w.r.t.~$b$ of the active action value function.
\end{proof}
Intuitively, the condition requires that the intervention effect on processes in the ``bad'' state must be large, making $P_{1,1}^a-P_{0,1}^a$ small. Note that \citet{zhao_liyu_paper} consider the case where $P_{1,1}^a = P_{1,1}^p$ and $P_{0,1}^a = P_{0,1}^p$, which makes Eq.~\ref{eq:final_forward_threshold_condition} always true. Thus we generalize their result for threshold optimality.

\begin{theorem}\label{thm:reverse_threshold_opt}
Consider a belief-state MDP corresponding to an arm in a Collapsing Bandit. For any subsidy $m$, there is a \emph{reverse} threshold policy that is optimal under the condition:
\begin{align}
    \label{eq:final_reverse_threshold_condition_maintext}
    (P_{1,1}^p-P_{0,1}^p)\Big(1+\frac{\beta (P_{1,1}^a-P_{0,1}^a) }{1-\beta}\Big)\le P_{1,1}^a-P_{0,1}^a
\end{align}
\end{theorem}
Intuitively, the condition requires small intervention effect on processes in the ``bad'' state, the opposite of the forward threshold optimal requirement. Note that both Thm.~\ref{thm:forward_threshold_opt} and Thm.~\ref{thm:reverse_threshold_opt} also serve as conditions for the average reward case as $\beta\rightarrow{}1$ (a proof based on Dutta's Theorem \cite{dutta1991discounted} is given in Appendix D).
%\ref{appendix:dutta}).

% \hf{I am not sure it is a good idea to put conjecture at this place. Maybe put them to the conclusion, or experimental section when you discuss experiment results,  and future work? Putting it here sounds a bit like you don't know how to solve a problem so just put it as a conjecture. }
\begin{conjecture}\label{conjecture:no_dual_thresh}
Dual threshold policies are never optimal for Collapsing Bandits.
\end{conjecture}
This conjecture is supported by extensive numerical simulations over the random space of state transition probabilities, values of $\beta$, and values of subsidy $m$; its proof remains an open problem. Note that this would imply that all Collapsing Bandits are indexable.

\section{Experimental Evaluation}
We evaluate our algorithm on several domains using both real and synthetic data distributions. We test the following algorithms: 
\noindent\textbf{Threshold Whittle} is the algorithm developed in this paper. \textbf{\citet{qian2016restless}}, a slow, but precise general method for computing the Whittle index, is our main baseline that we improve upon.
%, uses binary search and POMDP solvers to find the Whittle index. 
% \textbf{No intervention} always remains passive.
\textbf{Random} selects $k$ process to act on at random each round.
% \paragraph{Call everybody:} Make $k = N$ calls each day; this is an unattainable upper baseline.
\textbf{Myopic} acts on the $k$~processes that maximize the expected  reward at the immediate next time step. 
Formally, at time $t$, this policy picks the $k$ processes with the largest values of $\Delta b_{t} = (b_{t+1}|a=1)-(b_{t+1}|a=0)$.
\textbf{Oracle} fully observes all states and uses \citet{qian2016restless} to calculate Whittle indices. We measure performance in terms of \emph{intervention benefit}, where $0\%$ corresponds to the reward of a policy that is always passive and 100\% corresponds to Oracle. All results are averaged over 50 independent trials. 

\subsection{Real Data: Monitoring Tuberculosis Medication Adherence}

We first test on tuberculosis medication adherence monitoring data, which contains daily adherence information recorded for each real patient in the system, as obtained from \citet{killian2019learning}. The ``good'' and ``bad'' states of the arm (patient) correspond to ``Adhering'' and ``Not Adhering'' to medication, respectively. 
% Note that health workers only observe the state of patients they call, while Oracle here corresponds to the situation when all patients perfectly report their daily adherence via \emph{digital adherence technology (DAT)}. 
State transition probabilities are estimated from the data. Because this data is noisy and contains only the adherence records and not the intervention (action) information (as the authors state), we perturb the computed average transition matrix by reducing (increasing) $P_{\omega,1}$ by $\delta_1, \delta_2$ ($\delta_3, \delta_4$) to obtain $P_{\omega,1}^p$ ($P_{\omega,1}^a$) for the simulation. Reward is measured as the undiscounted sum of patients (arms) in the adherent state over all rounds, where each trial lasts $T=180$ days (matching the length of first-line TB treatment) with $N$ patients and a budget of $k$ calls per day. 
% In this setting, Oracle corresponds to the situation when all patients perfectly report their daily adherence via \emph{digital adherence technology (DAT)} \cite{subbaraman2018digital}. 
% We leave the straightforward extension of Threshold Whittle to handle noisy DAT reporting to future work, but offer Oracle as an upper bound of achievable adherence with full DAT information.

In Fig.~\ref{fig:runtime-plot}, we plot the runtime in seconds vs the number of patients $N$. Fig.~\ref{fig:tb-performance} compares the intervention benefit for $N=100, 200, 300, 500$ patients and $k=10\% $ of $N$. In the $N=200$ case, the runtimes of a single trial of Qian et al.~and Threshold Whittle index policy are $3708$ seconds and $3$ seconds, respectively, while attaining near-identical intervention benefit. Our algorithm is thus $3$ orders of magnitude faster than the previous state of the art without sacrificing performance.

We next test Threshold Whittle as the resource level~$k$ is varied. Fig.~\ref{fig:kplot} shows the performance in the  $k=5\%N$, $k=10\%N$ and $k=15\%N$ regimes ($N=200$). Threshold Whittle outperforms Myopic and Random by a large margin in these low resource settings.  
% E.g., when $k=5\%N$, the intervention benefits of Random, Myopic and Threshold Whittle are $16\%$, $32\%$ and $97\%$, respectively.
% We next affirm the robustness of our algorithm to $\delta$, the perturbation parameter used to approximate real-world $P_{\omega,1}^p$ and $P_{\omega,1}^a$ from the data. In Fig.~\ref{fig:perturbation_plot} we vary $\delta_3$ on the $x$-axis and measure the intervention benefit on the $y$-axis (plots of $\delta_1$, $\delta_2$ and $\delta_4$ are omitted for space). Threshold Whittle is robust to the perturbation parameters.
We also affirm the robustness of our algorithm to $\delta$, the perturbation parameter used to approximate real-world $P_{\omega,1}^p$ and $P_{\omega,1}^a$ from the data and present the extensive sensitivity analysis in Appendix G.  %\ref{appendix:sensitivity}. 
Finally, in Appendix F %\ref{appendix:online_learning}
we couple our algorithm to a Thompson Sampling-based learning approach and show it performs well in the real-world case where transition probabilities would need to be learned online, supporting the deployability of our work.

\begin{figure}[h!]
\centering
\subfloat[]{{\includegraphics[ width=0.28\textwidth, clip]{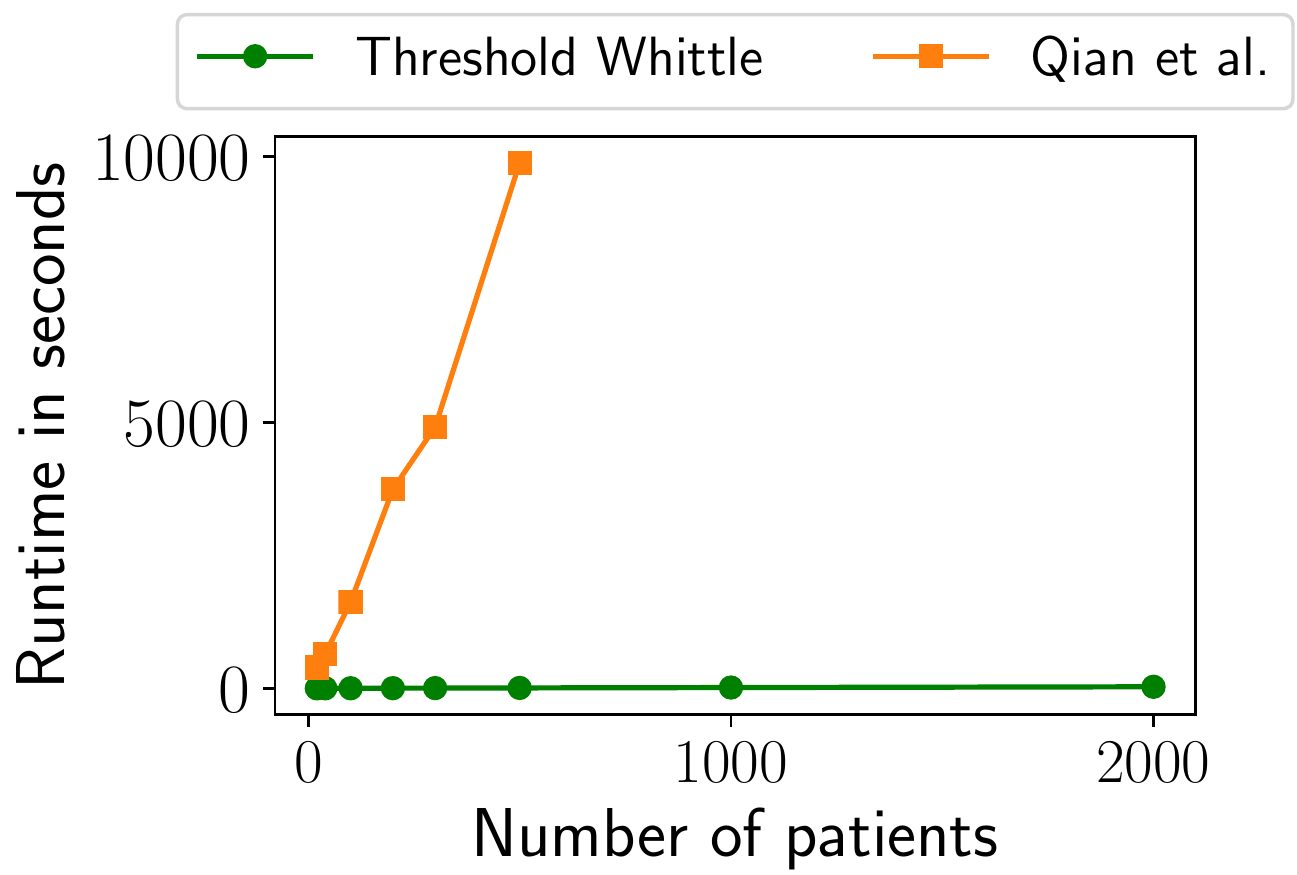}}\label{fig:runtime-plot}}
\hspace{2mm}
\subfloat[]{{\includegraphics[ width=0.35\textwidth, clip]{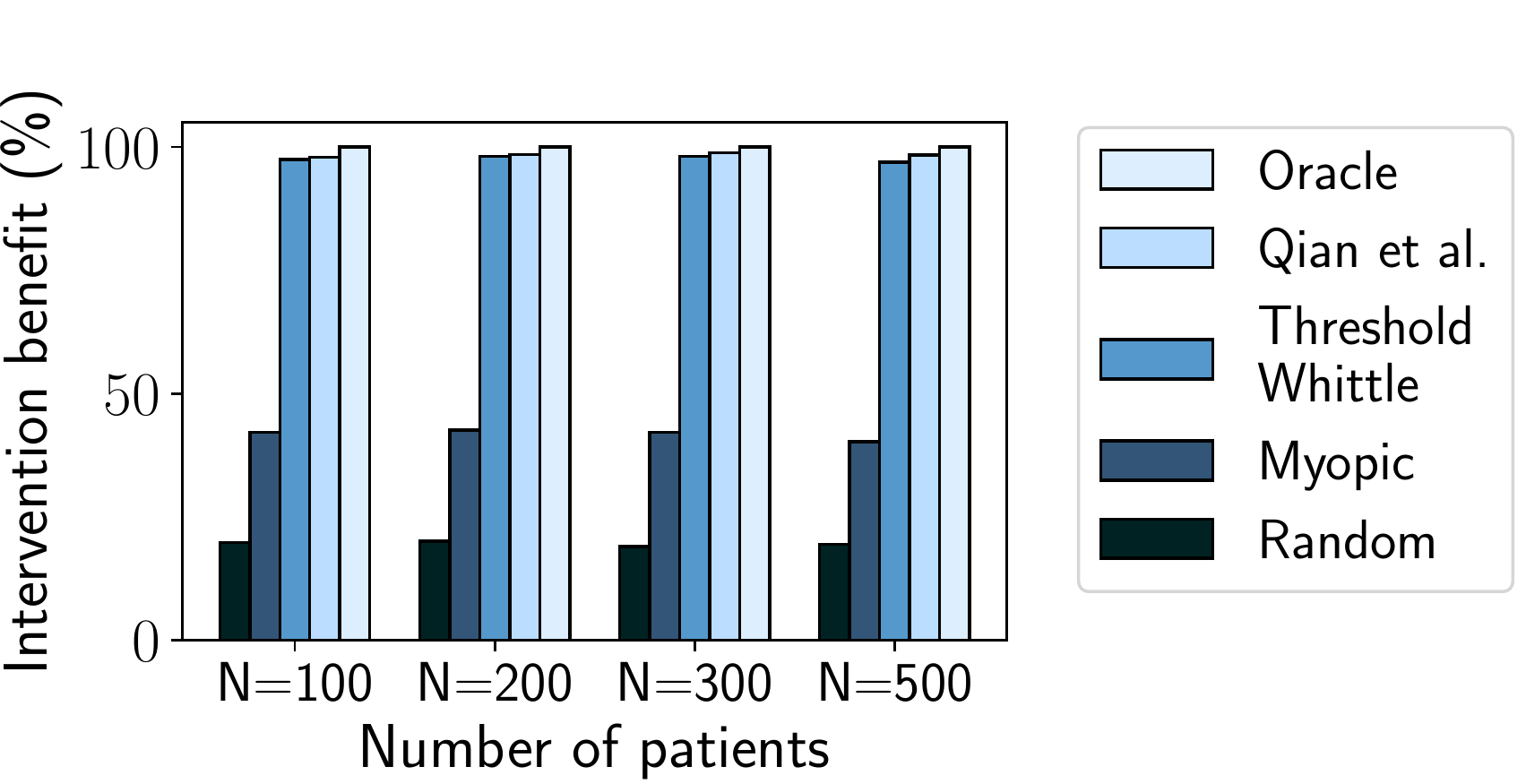}}\label{fig:tb-performance}}%
\hspace{2mm}
\subfloat[]{\includegraphics[width=.27\textwidth]{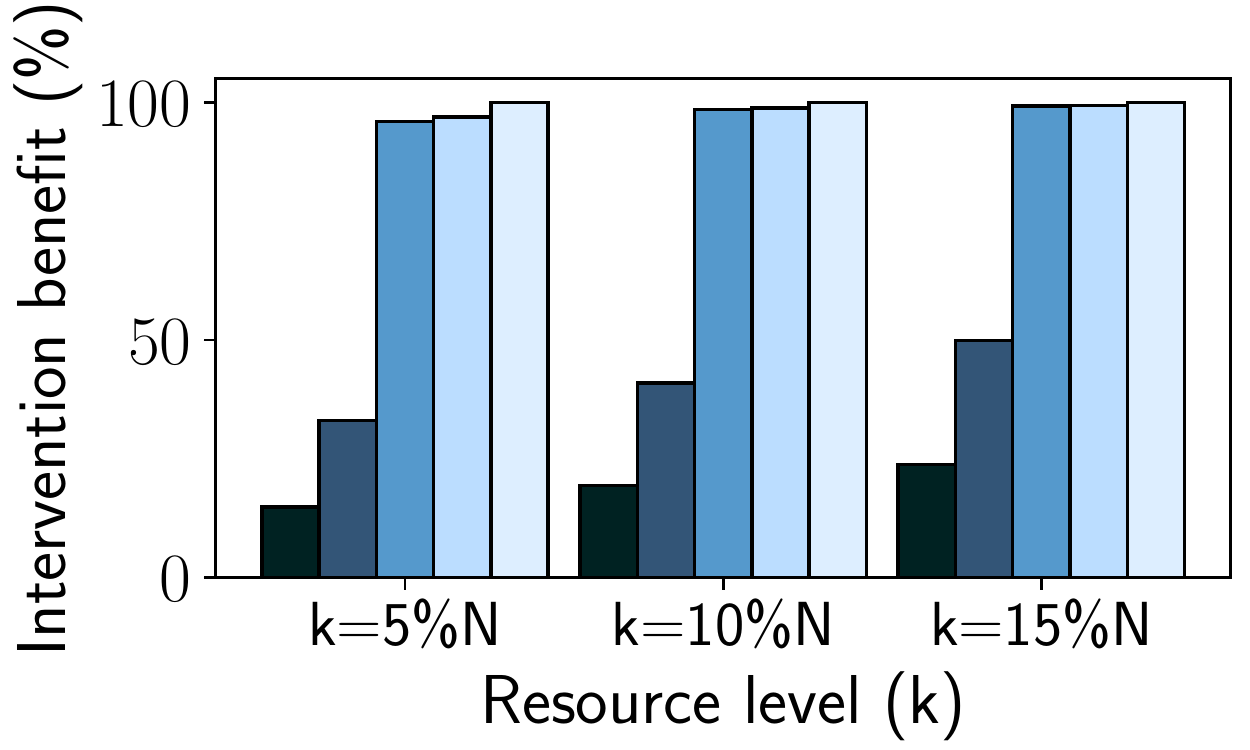}\label{fig:kplot}}
\caption{(a) Threshold Whittle is several orders of magnitude faster than Qian et al.~and scales to thousands of patients without sacrificing performance on realistic data (b). (c) Intervention benefit of Threshold Whittle is far larger than naive baselines and nearly as large as Oracle.}%
\end{figure}

\subsection{Synthetic Domains}

We test our algorithm on four synthetic domains, that potentially characterize other healthcare or relevant domains, and highlight different phenomena. Specifically, we: (i) identify situations when Myopic fails completely while Whittle remains close to optimal, (ii) analyze the effect of latent state entropy on policy performance, (iii) identify limitations of Threshold Whittle by constructing processes for which Threshold Whittle shows separation from Oracle, and (iv) test robustness of our algorithm outside of the theoretically guaranteed conditions. To facilitate comparison with the real data distribution, we simulate trials for $T=180$ rounds where reward is the undiscounted sum of arms in state $1$ over all rounds. We consider the space of transition probabilities satisfying the assumed natural constraints, as outlined in Sec.~\ref{section:problem_formulation}.

Fig.~\ref{fig:synth_results}a demonstrates a domain characterized by processes that are either self-correcting or non-recoverable. Self-correcting processes have a high probability of transitioning from state $0$ to $1$ regardless of the action taken, while non-recoverable processes have a low chance of doing so. We show that when the immediate reward is larger for the former than the latter, Myopic can perform even worse than Random.  That is because a myopic policy always prefers to act on the self-correcting processes per their larger immediate reward, while Threshold Whittle, capable of long-term planning, looks to avoid spending resources on these processes. In this regime, the best long-term plan is to always act on the non-recoverable processes to keep them from failing. Analytical explanation of this phenomenon is presented in Appendix E. %\ref{appendix:myopic_example}. 
We set the resource level, $k=10\%N$ in our simulation for Fig.~\ref{fig:synth_results}a. Note that performance of Myopic drops as the fraction of self-correcting processes becomes larger and reaches a minimum at $x=100\%-k=90\%$. Beyond this point, Threshold Whittle can no longer completely avoid self-correcting processes and the gap subsequently starts to decrease. 

Fig.~\ref{fig:synth_results}b explores the effect of uncertainty in the latent state on long-term planning. For each point on the $x$-axis, we draw all transition probabilities according to $P_{\omega,1}^p, P_{\omega,1}^a \sim [x,x+0.1]$. The entropy of the state of a process is maximum near 0.5 making long term planning most uncertain and as a result, this point shows the biggest gap with Oracle, which can observe all the states in each round. Note that Myopic and Whittle policies perform similarly, as expected for (nearly) stochastically identical arms.

Fig.~\ref{fig:synth_results}c studies processes that have a large propensity to transition to state $0$ when passive and a corresponding low active action impact, but a significantly larger active action impact in state $1$. This makes it attractive to exclusively act on processes  in the $1$ state. 
% Note that unlike non-recoverable processes, these have a significant $P_{01}^p, P_{01}^a$ and a much larger active action impact. 
This simulates healthcare domains where a fraction of patients  degrade rapidly, but can recover, and indeed respond very well to interventions if already in a good state.
To simulate these, we draw transition matrices with $P_{0,1}^p, P_{1,1}^p, P_{0,1}^a \sim [0.3,0.32]$ and $ P_{1,1}^a \sim [0.7,0.72]$ 
%(subject to the natural constraints on $P$ defined earlier)  
in varying proportions and sample the rest from the real TB adherence data. Because the best plan is to act on processes in state $1$, both Myopic and Whittle act on the processes with the largest belief giving Oracle a significant advantage as it has perfect knowledge of states. 

Although we provide theoretical guarantees on our algorithm for forward threshold optimal processes with non-increasing belief, Fig.~\ref{fig:synth_results}d reveals that Alg.~\ref{alg:algo1} performs well empirically even with these conditions relaxed. Here, we sample processes uniformly at random from the state transition probability space, and use rejection sampling to vary the proportion of threshold optimal processes. Threshold Whittle performs well even when as few as $20\%$ of the processes are forward threshold optimal; we briefly analyze this phenomenon in Appendix H. %\ref{appendix:alg1_reverse}.

\begin{figure}[h!]
\includegraphics[ width=.85\textwidth, clip]{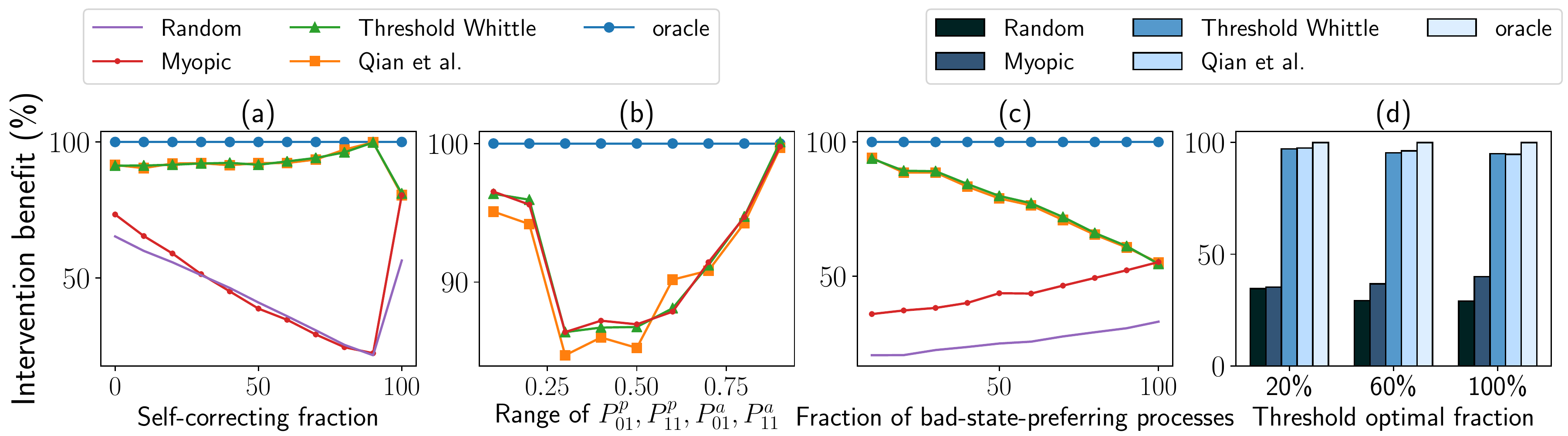}
\centering
\caption{(a) Myopic can be trapped into performing even worse than Random while Threshold Whittle remains close to optimal. (b) Long-term planning is least effective when entropy of states is maximum. (c) Myopic and Whittle planning become similar when more processes are prone to failures. (d) Threshold Whittle is surprisingly robust to processes even outside of theoretically guaranteed conditions.}
\label{fig:synth_results}
\end{figure}

\section{Conclusion}
We open a new subspace of Restless Bandits, \emph{Collapsing Bandits}, which applies to a broad range of real-world problems, especially in healthcare delivery. We give new theoretical results that cover a large portion of real-world data as well as an algorithm that runs thousands of times faster than the state of the art without sacrificing performance. 

\bibliographystyle{plainnat}
\bibliography{main.bib}
\clearpage

%\appendix
%\input{appendix}

\end{document}